\def\year{2021}\relax
\documentclass[letterpaper]{article} 
\usepackage{aaai21}  
\usepackage{times}  
\usepackage{helvet} 
\usepackage{courier}  
\usepackage[hyphens]{url}  
\usepackage{graphicx} 
\usepackage{natbib}  
\usepackage{caption} 
\usepackage{algorithm}
\usepackage[noend]{algpseudocode}
\usepackage{amsmath}
\usepackage{amsfonts}
\usepackage{amsthm}
\renewcommand{\algorithmiccomment}[1]{\hfill\textit{// #1}}

\usepackage{xcolor}
\usepackage{multirow}

\newtheorem{theorem}{Theorem}
\newtheorem{proposition}{Proposition}
\newtheorem{definition}{Definition}

\title{Hierarchical Width-Based Planning and Learning$^\text{\textbf{*}}$}
\author{
   Miquel Junyent, Vicen\c{c} G\'{o}mez, Anders Jonsson \\
}
\affiliations{
Universitat Pompeu Fabra \\
Barcelona, Spain \\
\{miquel.junyent, vicen.gomez, anders.jonsson\}@upf.edu
}

\begin{document}

\maketitle

\begin{abstract}
Width-based search methods have demonstrated state-of-the-art performance in a wide range of testbeds, from classical planning problems to image-based simulators such as Atari games. These methods scale independently of the size of the state-space, but exponentially in the problem width. In practice, running the algorithm with a width larger than 1 is computationally intractable, prohibiting IW from solving higher width problems. In this paper, we present a hierarchical algorithm that plans at two levels of abstraction. A high-level planner uses abstract features that are incrementally discovered from low-level pruning decisions. We illustrate this algorithm in classical planning PDDL domains as well as in pixel-based simulator domains. In classical planning, we show how $\text{IW}(1)$ at two levels of abstraction can solve problems of width 2. For pixel-based domains, we show how in combination with a learned policy and a learned value function, the proposed hierarchical IW can outperform current flat IW-based planners in Atari games with sparse rewards.
\end{abstract}

\section{Introduction}
The use of hierarchies in planning has proven to be a very successful way for significantly reducing the computational cost of finding good plans. 
Traditional methods include Hierarchical Task Networks~\cite{currie1991plan,erol1996complexity}, macro-actions~\cite{fikes1972learning,korf1985macro}, and state abstraction methods~\cite{sacerdoti1974planning,knoblock}.
Hierarchical planning can lead to exponential gains in complexity by exploiting the structure of a problem involving a reduced subset of the state components.

Iterated Width (IW)~\cite{iw2012ecai} is a search algorithm that makes use of the feature representation of the states to perform structured exploration. The original IW algorithm consists of successive breadth-first searches in which states are pruned if they fail to meet a novelty criterion. In particular, $\text{IW}(w)$ only considers $w$ features at a time, and prunes those states for which all combinations of $w$ features are made true in previously generated states. $\text{IW}(w)$ runs in time and space that are exponential in $w$, but independent of the size of the state space.

Initially proposed as a blind search method for classical planning, IW search has been extended in many different ways, resulting in several competitive width-based planners, including LW1 for partially observable domains~\cite{belief}, or BFWS as an informed (best-first) width search planner~\cite{lipovetzky2017Best-FirstPlanning}.

One particular advantage of width-based planners is that, unlike other classical planners, they do not need a declarative representation of actions, costs or goals~\cite{ijcai2017-600}. Width-based planners are thus directly applicable in simulator environments, achieving state-of-the-art performance in the General Video Game competition~\cite{geffner2015width} and the Atari suite~\cite{ijcai15atari,shleyfman2016blind,bandres2018planning}.

The performance of IW strongly depends on how informative the state features are. Using poorly informed features requires a large value of $w$ to reach a goal state, whereas using highly informative features reduces the problem width and, hence, makes it solvable using a lower value of $w$. This effect is known, e.g., in Atari, where using informative RAM states leads to better results than planning directly with pixels~\cite{bandres2018planning}. How to discover or learn such features to reduce the problem width is an open problem, and several ideas have been proposed, including the use of conjunctive features~\cite{ijcai2017-600} or deep learning methods~\cite{junyent2019,PRL20}.

In practice, IW is mostly used with $w=1$ with complexity linear in the number of features~\cite{geffner2015width,bandres2018planning,uav,PRL20}. In many challenging problems, even $w=2$ with quadratic complexity is unfeasible~\cite{geffner2015width}.
Finding ways to run IW with a larger value of $w$ can further extend the applicability of this class of planners.

In this work, we propose a hierarchical formulation of width-based planning that takes advantage of both the structured search performed by width-based algorithms as well as the concept of hierarchy, which captures explicitly the idea of using state abstraction to reduce effectively the width of a problem. The framework can be combined with other forms of learning to further extend the applicability of width-based planners.

\section{Background}

In this section we define Markov decision processes and the Iterated Width (IW) algorithm, and introduce notation that will be used throughout the paper.

\subsection{Markov Decision Processes}

A Markov decision process (MDP) is modeled as a tuple $M = \langle S,A,P,r \rangle$, where $S$ is a finite set of states, $A$ is a finite set of actions, $P$ is a transition function and $r$ is a reward function. We assume that the transition function $P$ is {\em deterministic}, i.e.,~$P:S\times A\rightarrow S$ maps state-action pairs to next states, while the reward function $r:S\times A\rightarrow\mathbb{R}$ maps state-action pairs to real-valued rewards.

At each time step $t$, a learning agent observes state $s_t\in S$, selects an action $a_t\in A$, transitions to a new state $s_{t+1}=P(s_t,a_t)$ and receives reward $r_t = r(s_t,a_t)$. The aim of the learner is to compute a policy $\pi:S\rightarrow\Delta(A)$, i.e.,~a mapping from states to probability distributions over actions, that maximizes some measure of expected future reward. Here, $\Delta(A)=\{\mu\in\mathbb{R}^{|A|}: \sum_a\mu(a)=1, \mu(a)\geq 0 \; (\forall a)\}$ is the probability simplex over $A$.

The expected future reward associated with policy $\pi$ is governed by a value function $V^\pi$, defined in each state $s$ as
\[
V^\pi(s) = \mathbb{E}_\pi \left[\sum_{t=0}^\infty \gamma^t r(S_t,A_t) \Bigg\vert S_0 = s\right].
\]
Here, $S_t$ and $A_t$ are random variables representing the state and action at time $t$, respectively, satisfying $A_t \sim \pi(S_t)$ and $S_{t+1} = P(S_t,A_t)$ for each $t\geq 0$, and $\gamma\in(0,1]$ is a discount factor. The {\em optimal value function} $V^*$ is given by $V^* = \max_\pi V^\pi$, and the optimal policy $\pi^*$ is the argument achieving this maximum, i.e.,~$\pi^* = \arg\max_\pi V^\pi$.

We assume that there exists a set of features $F$, each with finite domain $D$, and a mapping $\phi:S\rightarrow D^{|F|}$ from states to feature vectors. For each feature $f\in F$ and state $s\in S$, let $\phi(s)[f]\in D$ be the value that $s$ assigns to $f$. It is common to approximate the value function in state $s$ using the feature vector $\phi(s)$ and a parameter vector $\theta$, i.e.,~the estimation of the value in state $s$ is given by $\hat{V}_\theta(s) = g(\phi(s),\theta)$ for some function $g$, e.g.~a neural network.

We can use deterministic MDPs to model goal-directed planning tasks. Such a planning task is also defined by a set of states $S$, a set of actions $A$ and a deterministic transition function $P$. In addition, there is a set of designated goal states $S_G\subset S$. To model the task as an MDP, we make each goal state $s_G\in S_G$ absorbing by defining the transition function as $P(s_G,a)=s_G$ for each action $a\in A$. The reward function is defined as $r(s,a)=1$ if $P(s,a)\in S_G$ and $r(s,a)=0$ otherwise. Hence an optimal policy attempts to reach a goal state as quickly as possible and then stay there.

\subsection{Iterated Width}

Iterated Width (IW)~\cite{iw2012ecai} is a forward search algorithm that explores the state space of a deterministic MDP starting from a given initial state $s_0$. IW was initially developed for goal-directed planning tasks, attempting to find a goal state among the set of explored states. However, the algorithm has later been adapted to deterministic MDPs by instead attempting to maximize expected future reward~\cite{ijcai15atari}.

In its basic form, IW is a blind search algorithm that performs breadth-first search in the space of states, starting from $s_0$. However, unlike standard breadth-first search, IW uses a novelty measure to prune states. The novelty measure critically relies on the feature vector $\phi(s)$ associated with each state $s$. Concretely, IW defines a width parameter $w$, and remembers all visited tuples of feature values of size $w$ in a so-called \emph{novelty table}. During search, a state $s$ is considered novel if its associated feature vector $\phi(s)$ contains at least one tuple of feature values of size $w$ that has not been visited before. IW then prunes all states that are not novel.

For a given width $w$, because of pruning, the number of states visited by $\text{IW}(w)$ is exponential in $w$. Since the state space is usually large, $\text{IW}(w)$ is typically provided with a search budget, and terminates when the number of visited states exceeds the budget. Without a search budget, in most domains it is computationally infeasible to execute $\text{IW}(w)$ for $w>2$. However, many planning benchmarks turn out to have small width, at least when considering atomic goals, and in practice they can be solved by $\text{IW}(1)$ or $\text{IW}(2)$.

Several researchers have proposed extensions to IW. Rollout IW~\cite{bandres2018planning} simulates a breadth-first search by repeatedly generating trajectories, or rollouts, from the initial state $s_0$. This is useful in domains for which it is expensive to store states in memory, making it impractical to perform an actual breadth-first search. The $\pi\text{-IW}$ algorithm~\cite{junyent2019} maintains and updates a policy $\pi$, and uses the policy to decide in which order to expand states, rather than exploring blindly.

\section{Complexity of IW($w$)}

In this section we provide a tighter upper bound on the number of states visited by $\text{IW}(w)$. We use $n=|F|$ to denote the number of features, and $d=|D|$ to denote the domain size. We also assume that at most $b$ actions are applicable in each state $s$. In~\citet{ijcai15atari} it was shown that $\text{IW}(w)$ generates at most $b(nd)^w$ nodes.


\begin{proposition}
Let $N(n,d,w)$ denote the maximum number of {\em novel} states visited by $\text{IW}(w)$ for a given pair $(n,d)$. Then, $N(n,d,w)$ is given by the recursive formula
\begin{align*}
    N(n,d,0) &= 1,\\
    N(n,d,n) &= d^n,\\
    N(n,d,w) &= (d-1) N(n-1,d,w-1) + N(n-1,d,w).
\end{align*}

\end{proposition}

Given $N(n,d,w)$, the number of visited states (including those pruned) is bounded by $N(n,d,w)\cdot b$. There are two base cases: $w=0$, in which case no state is novel apart from $s_0$, i.e.,~$N(n,d,0)=1$, and $w=n$, in which case all states are novel, i.e.,~$N(n,d,n)=d^n$.

The intuition for the recursion is as follows. 
Consider the case where $\text{IW}(w)$ visits the maximum number of states.
Given a feature $f \in F$, we can partition the subset of novel states into two subsets: states that are novel solely due to tuples that include $f$, denoted by $S_f$, and states that are novel (in part) due to tuples that exclude $f$, denoted by $S_{\neg{f}}$. 

Since $f$ is irrelevant in $S_{\neg{f}}$, $\text{IW}(w)$ would generate the same novel states even if we removed $f$. Thus, the maximum amount of novel states in $S_{\neg{f}}$ is bounded by $N(n-1,d,w)$. Regarding $S_f$, we can divide it into $d-1$ subsets, each corresponding to a value of $f$ different from its initial value $v_0=\phi(s_0)[f]$. In each subset, since the value of $f$ is the same, 
the novelty test can be simplified to checking tuples of size $w-1$ of features different than $f$. Therefore, the maximum number of novel states in $S_f$ is $(d-1) \cdot N(n-1,d,w-1)$.

We provide an example in Table \ref{tab:N-rec-example}, in the supplementary material. Note that we are not decomposing the problem into multiple subproblems; rather, the recursion defines an upper bound on the number of novel states in each subset. 



\begin{theorem}\label{thm:N}
For $n$ features of size $d$, the maximum number of novel states visited by $\text{IW}(w)$, $0\leq w<n$, is
\[
	N(n,d,w) = \sum_{k=0}^{w}\left[ \binom{n-1-k}{w-k} d^k (d-1)^{w-k} \right].
\]
\end{theorem}
The proof of Theorem~\ref{thm:N} appears in the supplementary material, and also shows that $N(n,d,w)$ is indeed upper bounded by $(nd)^w$, which is consistent with previous results.

\section{Hierarchical IW}
In this section, we present our hierarchical approach to width-based planning. We start by defining a simple algorithm for hierarchical blind search. Then, we consider using width-based planners at all levels of the hierarchy, and show its effect on the width compared to planning at a single level.

For simplicity, WLOG we assume a two-level hierarchy: a high level ($h$) and a low level ($\ell$). 
Each level 
is defined by its own feature set ($F_h$ and $F_\ell$, with domains $D_h$ and $D_\ell$, respectively) and feature mapping ($\phi_h:S\rightarrow D_h^{|F_h|}$ and $\phi_\ell:S\rightarrow D_\ell^{|F_\ell|}$, respectively). Each state $s$ maps to a high-level state $s_h=\phi_h(s)$ and a low-level state $s_\ell=\phi_\ell(s)$.

\subsection{A Hierarchical Approach to Blind Search}
Blind search methods require two components: a successor function, that given a state and an action returns a successor state (e.g. a simulator), and a stopping condition, that will stop the search, for instance, when the goal is reached or after a budget is exhausted. In order to have different search levels, we modify these two components as follows:
\begin{itemize}
    \item \textbf{High-level successor function:} Each call to this function triggers a low level search, that runs until a new high-level state is found (i.e.,~a state $s$ that maps to a different $\phi_h(s)$).
    \item \textbf{Low-level stopping condition:} When a different high-level state is encountered, the search is stopped, returning control to the high-level planner. This stopping condition is added to the existing stopping conditions.
\end{itemize}

The control goes back and forth between the high and low-level planners. Each time that the high-level successor function is called, the according low-level search is resumed, generating new states until a new high-level state is found. We achieve this by storing a low-level search tree for each high-level state. If the low-level search terminates without finding a new high-level state, the high-level successor function returns \textit{null}, and the high-level state is marked as \textit{expanded}. The high-level planner will only generate successors from non-expanded high-level states, and can resume search from any state by retrieving it from memory.

The proposed framework allows many levels of abstraction, as well as the possibility to have different search methods at each level. For instance, we could have a breadth-first search at the high level and depth-first search at the low level, or combine different width-based search methods.

\subsection{Hierarchical Width}
The framework in the previous section partitions the states into subsets based on high-level features. To plan over the subsets, we can use any width-based search method as a high-level planner. For instance, we can apply $\text{IW}(2)$ at the high level and $\text{IW}(1)$ at the low level. We denote this by $\text{HIW}(2,1)$.
We next define a type of high-level feature that we call {\em splitting}, and compare HIW with flat IW, showing the effect of the hierarchy on the width of the problem.
\begin{definition}
A high-level feature $f\in F_h$ is splitting if, for each value $v\in D_h$, the induced subset of states $\{s\in S: \phi_h(s)[f] = v\}$ is a connected graph.
\end{definition}

\noindent\textbf{Example}: consider a simple problem where an agent needs to move along a corridor of length $L$, pick up a key, and go back along the same path to open a door. We can describe this problem using two features: $p$ (the position) and~$k$ (whether or not the key is held). Initially $p=0$ and $k=0$. The goal is $p=0$ and $k=1$. If $k\in F_h$, then $k$ is splitting: when $k$ is false, the agent can still visit all the positions of the corridor, and likewise when $k$ is true.

\begin{theorem}
\label{thm:HIW}
If all features in $F_h$ are splitting, $\text{HIW}(w_h, w_\ell)$ is equivalent to a restricted version of $\text{IW}(w_h+w_\ell)$ with tuples of $w_h$ features from $F_h$ and $w_\ell$ features from $F_\ell$.
\end{theorem}

\begin{proof}
Since each feature in $F_h$ is splitting, when we apply $\text{IW}(w_\ell)$ in a high-level state $s_h$, the subset of states induced by $s_h$ is connected. Since the restricted version of $\text{IW}(w_h+w_\ell)$ considers exactly $w_\ell$ features in $F_\ell$, it will explore the same low-level states as $\text{IW}(w_\ell)$. At the high-level, the restricted version of $\text{IW}(w_h+w_\ell)$ considers exactly $w_h$ features in $F_h$, so it will explore the same high-level states as $\text{IW}(w_h)$. Since the tuples in $\text{IW}(w_h+w_\ell)$ involve features in both $F_h$ and $F_\ell$, each state in the low-level search of a new high-level state is novel. Hence $\text{HIW}(w_h, w_\ell)$ explores the same states as the restricted version of $\text{IW}(w_h+w_\ell)$.
\end{proof}

\noindent\textbf{Example (cont.)}: The corridor example has width $2$, since IW needs to keep track of the key and visited position jointly. This example can be solved by $\text{HIW}(1,1)$ using $F_h = \{k\}$ and $F_\ell = \{p\}$, after two low-level searches (one for $k=0$ and one for $k=1$), and visits the same states as IW$(2)$. 

Theorem \ref{thm:HIW} compares $\text{HIW}(w_h,w_\ell)$ to flat $\text{IW}(w_h + w_\ell)$ when all the features in $F_h$ are splitting. However, this is not a necessary condition for $\text{HIW}(w_h,w_\ell)$ to solve problems of width $w_\ell + w_h$. Without splitting features, HIW will not generate the same nodes as the restricted version of IW, but may still find the goal. We empirically show this in the experiments section.

\begin{theorem}\label{thm:HIW_comp}
Let $n_h=|F_h|$ and $d_h=|D_h|$ be the number of high-level features and domain sizes, and define $(n_\ell,d_\ell)$ analogously. The maximum number of novel states expanded by $\text{HIW}(w_h, w_\ell)$ is $N(n_h, d_h, w_h) \cdot N(n_\ell,d_\ell,w_\ell)$.
\end{theorem}

\begin{proof}
At the high level, $\text{HIW}(w_h, w_\ell)$ applies $\text{IW}(w_h)$, which expands a maximum of $N(n_h, d_h, w_h)$ novel high-level states due to Theorem~\ref{thm:N}. For each novel high-level state, $\text{HIW}(w_h, w_\ell)$ applies $\text{IW}(w_\ell)$, which expands a maximum of $N(n_\ell,d_\ell,w_\ell)$ novel low-level states.
\end{proof}

\noindent Note that the maximum number of novel states expanded by the unrestricted version of $\text{IW}(w_h+w_\ell)$ on the feature set $F=F_h \cup F_\ell$ is $N(n_h+n_\ell, \max(d_h,d_\ell), w_h+w_\ell)$, which is much larger than $N(n_h, d_h, w_h) \cdot N(n_\ell,d_\ell,w_\ell)$ in general.

\noindent\textbf{Example}: The RAM memory in Atari, used in \citet{ijcai15atari}, consists of $n=128$ features with $d=256$ values. For $\text{IW}(2)$, an upper bound on the number of novel states is $N(n,d,w) \sim 5 \cdot 10^8$. If we identify a splitting feature and define $n_h=1$, $n_\ell=127$, and $w_h=w_\ell=1$, the upper bound due to Theorems 2 and 3 is $N(n_h,d,w_h) \cdot N(n_\ell,d,w_\ell) \sim 8 \cdot 10^6$, an improvement of almost two orders of magnitude.

\section{Incremental Hierarchical IW (IHIW)}
In classical planning, the states are defined by a set of atoms, and, although one atom may be more informative than others, there is no hierarchical structure. In this section, we present a simple method for identifying relevant features that may split the state space. Then, we introduce an algorithm that performs a sequence of hierarchical searches, using the aforementioned method to discover new high-level feature candidates at each step. In the experiments section, we test the algorithm in a range of classical planning domains\footnote{The code for all algorithms and experiments described in this paper can be found in https://github.com/aig-upf/hierarchical-IW.}.

\subsection{Discovering High-Level Features}
Consider a search tree generated by $\text{IW}(1)$ for a problem of width $2$. Is it possible to identify features that split the state space, so that the problem can be solved by $\text{HIW}(1,1)$? In this section, we present a simple method for detecting candidate abstract features from a set of features $F$.

We consider all trajectories in the tree and hypothesize that a feature that changes only once before a trajectory is pruned is a good candidate for a high-level feature. Consider again the corridor example in which an agent has to use a key to open a door. $\text{IW}(1)$ prunes any trajectory that repeats a position $p$, and will not solve the problem. However, feature~$k$ splits the state space into two sub-problems: reaching the key ($k$=true), and going back to the door ($k$=false).

We can detect high-level features using the method detailed in Algorithm \ref{algo:findAbstract}. 
For each pruned leaf node, we retrieve the features that are shared with its parent that have not appeared in that branch before. The intuition is that when a splitting feature $f$ changes value, from $v_0$ to $v_1$, the {\em next} state is likely to be pruned by $\text{IW}(1)$, since $v_1$ has just been observed for $f$, and all other features may have been visited when $f$ took value $v_0$.

\begin{algorithm}[t] 
\caption{Method for finding high-level features}
\label{algo:findAbstract}
\begin{algorithmic}
    \State \textbf{Input:} node $n$
	\State $N = \emptyset$
	\If {IsLeaf($n$) \& Depth($n$) $> 2$} 
		\State $P$ = Atoms($n$) $\cap$ Atoms(Parent($n$)) \hfill
		\Comment{common atoms}
		\If{$|P| < |\textnormal{Atoms}(n)|$} \Comment{ensure different state}
			\State $b = \textnormal{Branch}(tree, n)$ \Comment{get branch root$\rightarrow$n}
			\State $B = \bigcup\limits_{i=1}^{\textnormal{Depth}(n)-2} \textnormal{Atoms}(b[i])$ \Comment{all branch atoms}
			\State $N = P - B$ \Comment{keep (branch) novel atoms}
		\EndIf
	\EndIf
	\State \textbf{return} $N$
\end{algorithmic}
\end{algorithm}

\subsection{An Incremental Approach}
A simple algorithm that takes advantage of the previous method would be:
\begin{enumerate}
    \item Perform an $\text{IW(1)}$ search, if the goal is found, return.
    \item Run Alg.~\ref{algo:findAbstract} on the $\text{IW}(1)$ tree to find high-level features.
    \item Run $\text{HIW}(1,1)$ with the discovered high-level features.
\end{enumerate}

\begin{algorithm}[t]
\caption{Incremental Hierarchical IW Search}
\label{algo:serializedHIW}
\begin{algorithmic}
    \State \textbf{Initialize:} $H = \emptyset$, $P$ = List(), \textit{solved} = \textbf{false}
	\While{\textbf{not} \textit{solved}}
	    \State \textit{pruned}, \textit{solved} = HIW($w_h$, $w_l$)
	    \If{\textbf{not} \textit{solved}}
            \State Append($P$, \textit{pruned})
    		\While{$H == \emptyset$}
    		    \If{$P$ is empty}
    		        \State \textbf{return}
    		    \EndIf
    		    \State $n$ = Pop($P$) \algorithmiccomment{Sample pruned node}
    		    \State $H$ = FindAbstractFeatures($n$) \algorithmiccomment{Algorithm \ref{algo:findAbstract}}
    		\EndWhile
    		\State $h$ = Pop(H) \algorithmiccomment{Sample candidate atom}
    		\State RestructureTree($h$) \algorithmiccomment{Create high-level nodes}
    	\EndIf
	\EndWhile
\end{algorithmic}
\end{algorithm}

This algorithm actually finds promising candidate features for small problems. For instance, it can solve the simple corridor example. However, it fails on more complex problems, possibly because a single $\text{IW}(1)$ search may not be sufficient to visit states that contain relevant features.

To address this, we propose a slightly more sophisticated approach, Incremental HIW (Algorithm \ref{algo:serializedHIW}), that runs a series of HIW searches. It maintains a set of high-level feature candidates $H$, exploits one feature candidate at a time, and discovers new relevant features when necessary. First, we run $\text{HIW}(1,1)$, which is equivalent to $\text{IW}(1)$ since we start with $H=\emptyset$. While the task is not solved, we randomly sample a pruned node and update $H$ using Algorithm~\ref{algo:findAbstract}. We may repeat this operation until new feature candidates are found or there are no more pruned nodes to sample from, in which case we stop the search. Then, a feature candidate is sampled from $H$, and the current search tree is restructured accordingly, in order to reuse the tree in the subsequent search.

Restructuring the tree mainly involves two operations: detaching subtrees at the low level and inserting new nodes at the high level. Although this may seem costly, both operations consist of modifying the data structure, while leaving the data untouched. Modifying a search tree, however, implies that the associated novelty table cannot be reused. Thus, we generate a new novelty table, if necessary, when the according tree search is resumed.

\section{Learning with Hierarchy}

In this section we show how to combine HIW with a learning-based approach that uses a policy to direct search.

\subsection{Count-Based Rollout IW}
\citet{bandres2018planning} presented Rollout IW (RIW), a width-based algorithm
that performs breadth-first search implicitly, from independent rollout trajectories. $\text{RIW}(w)$ maintains the notion of width by modifying the definition of novelty: a state $s$ is considered novel if any $w$-tuple of features of $s$ has not appeared at a lower depth. With this, the authors achieve an algorithm that is equivalent to $\text{IW}(w)$, but with better anytime behavior. This novelty measure actually allows for many width-based algorithms, since it unties the order of expanding nodes from the novelty test.

In our scenario, a subset of states is encapsulated under the same high-level state (i.e., a set of high-level features). Selecting one high-level state or another directly determines which low-level states are generated. In order to balance exploration within high-level states, we extend RIW with a selection method that depends on state visitation counts.

Our method, named Count-based Rollout IW, is detailed in Algorithm \ref{algo:countbasedRIW}. It requires an OPEN list $O$, a mapping from features to counts $C$, a mapping from feature tuples to unpruned nodes $N$, and a depth-based novelty table $D$, which are all empty at the beginning.
Similar to RIW, it consists of two phases: node selection and rollout. A node from $O$ is selected according to a softmax probability distribution inversely proportional to the visitation counts of its feature vector. Then, a rollout is performed until a node that does not pass the novelty test is found. In  this  case, the novelty test function returns the set of novel tuples $T$.

For each new novel node $n$, with novel tuples $n.T$, there may be other nodes deeper in the tree that were initially novel due to one or more tuples of $n.T$, which may need to be pruned. We identify such nodes with the mapping $N$. Then, for each tuple $t \in n.T$, we can check which other node $o=N[t]$ was novel due to $t$, and remove $t$ from its set of novel nodes $o.T$ before setting $N[t]=n$. In the case the set becomes empty, the node should be pruned (i.e., removed from the open list, together with its descendants). This is done in function \textit{PruneOther}.

\begin{algorithm}[t]
\caption{Count-Based Rollout IW}
\label{algo:countbasedRIW}
\begin{algorithmic}

\Function{Select}{$O$, $C$}
    \State $c$ = GetCounts($O$, $C$) \Comment{Feature counts of nodes in $O$}
    \State $p \propto \exp\left(1/\tau(c+1)\right)$
    \State $n$ = Sample($O$, $p$)
    \State \textbf{return} $n$
\EndFunction
\vspace{5pt}
\Function{Rollout}{$n$, $O$, $N$, $C$, $D$}
\While{\textbf{not} StopCondition()}
        \State $C$[$n$.\textit{features}]++
        \State $s$ = Successor($n$)
        \If{s == \textit{null}}
            \State Remove($n$, $O$) \Comment{Remove $n$ from OPEN list}
            \State \textbf{return}
        \EndIf
        \State $s.T$ = Novel($D$, $s$.\textit{features}, $s$.\textit{depth})
        \If{IsEmpty($s.T$) \textbf{or} IsTerminal($s$)}
            \State \textbf{return}
        \EndIf
        \State PruneOther($s$, $O$, $N$)
        \State Append($O$, $s$) \Comment{Add node to OPEN list}
        \State $n$ = $s$
    \EndWhile
\EndFunction
\end{algorithmic}
\end{algorithm}

\subsection{Modifications to $\pi\text{-IW}$}
\citet{junyent2019} introduced Policy-Guided IW ($\pi\text{-IW}$), an on-line replanning algorithm that alternates planning and learning. $\pi\text{-IW}$ learns a policy $\pi$ from the rewards observed in the IW tree, and uses $\pi$ to guide future searches.
However, in sparse-reward tasks, $\text{IW}(1)$ may not reach any reward, especially when the planning horizon is too short.
Here we extend the original $\pi\text{-IW}$ in two ways: adding a better tie breaking mechanism, and
a value function estimate. In experiments, we call this (flat) version $\pi\text{-IW+}$.

When no reward is found during planning, the target policy for the learning step becomes the uniform distribution, and $\pi\text{-IW}$ behaves as Rollout IW. In this case, $\pi\text{-IW}$ may take a step towards a region of the search tree with low node count, and presumably with less novel states, losing valuable structure information provided by the IW search. To avoid that, we modify the behavior policy of $\pi\text{-IW}$ to use the node counts in the search tree for tie-breaking (i.e., the amount of descendants per action at the root node). The new behavior policy takes the form $\pi^{\textnormal{b}} \propto \pi^{\textnormal{rewards}} \cdot \pi^{\textnormal{counts}}$, where the product is element-wise, and $\pi^{\textnormal{counts}}$ is a softmax distribution:
\[
    \pi^{\textnormal{counts}}{(a|s)} \propto \exp{(1/(\tau c(s,a) + 1))},
\]
where $\tau$ is a temperature parameter and $c(s,a)$ is the amount of nodes in the subtree of action $a$.
The temperature parameter for $\pi^{\textnormal{rewards}}$, which is also defined as a softmax distribution but proportional to the returns $R(s,a)$, is typically close to zero to ensure a greedy target policy~\cite{junyent2019}. Therefore, by performing the product, we achieve the effect of tie-breaking, especially if the temperature parameter for the counts is some orders of magnitude higher than the one for the rewards.

This tie-breaking may help finding deeper rewards. However, $\pi\text{-IW}$ will not exploit this information in subsequent episodes, since $\pi^\textnormal{rewards}$ is still based on the rewards of the \emph{current} planning horizon. To amend this, we learn a value function, which we combine with the observed rewards to generate a better estimate of $\pi^{\textnormal{rewards}}$. When backpropagating the rewards from the leaves to the root, we take the maximum between the observed rewards and our value estimate.

To learn a parameterized policy estimate $\widehat\pi_\theta$, we follow the same approach of \citet{junyent2019}. Specifically, we represent $\widehat\pi_\theta$ using a neural network, and at each time step $t$, we use the cross-entropy loss to update $\theta$:
\begin{align*}
\mathcal{L}=-\pi_{t}^{\textnormal{rewards}}(\cdot|s_t)^\top\log\widehat\pi_\theta(\cdot|s_t).
\end{align*}
The difference in our work is that we additionally learn a value function, taking the same approach as in MuZero~\cite{schrittwieser2019mastering}, but using the Monte-Carlo return target. We also add an $\ell$-2 regularization term.


\subsection{Policy-Guided Hierarchical IW ($\pi\text{-HIW})$}
Hierarchical IW can be straightforwardly used for online replanning. At each step, we sample an action $a \sim \pi^{\textnormal{b}} \propto \pi^{\textnormal{rewards}} \cdot \pi^{\textnormal{counts}}$. To generate $\pi^{\textnormal{rewards}}$, we need to backpropagate the rewards through the hierarchical tree. Starting from the high-level leaf nodes, we first backpropagate the rewards of the associated low-level trees. Then, to propagate this return between two high-level nodes, we feed it to the corresponding low-level leaf nodes of the high-level parent, and repeat until we reach the high-level root. To generate $\pi^{\textnormal{counts}}$, we backpropagate the counts in a similar manner.

After executing an action $a$, we cache the resulting subtree for subsequent searches, similar to previous work. In this case, we need to take into account that some high-level states will not be reachable anymore, and we should thus remove them from the high-level tree before resuming the search.

\section{Experiments in Classical Planning}
In this section, we evaluate experimentally the proposed hierarchical approach. We address the following questions:
\begin{itemize}
\item In practice, can $\text{HIW}(1,1)$ solve problems of width $2$?
\item Can Algorithm~\ref{algo:findAbstract} find good high-level feature candidates?
\item Is $\text{IHIW}(1,1)$ a good alternative to $\text{IW}(2)$?
\end{itemize}

\citet{iw2012ecai} empirically showed that most classical planning problems with atomic goals present a low width. In Table~\ref{tab:pddl}, we reproduce such results, and compare them to our algorithm. The table consists of 36 domains from the International Planning Competitions, prior to 2012. For each domain, we show the amount of single goal instances (I), generated by splitting each instance with $G$ goal atoms into $G$ single goal instances. Columns 3-11 show the amount of instances solved, together with the average number of nodes and time per solved instance, for $\text{IW}(1)$, $\text{IW}(2)$ and $\text{IHIW}(1,1)$. Here, $\text{IHIW}(1,1)$ consists of two standard $\text{IW}(1)$ searches, one at each level of abstraction.

In some domains, $\text{IW}(1)$ has greater coverage than $\text{IW}(2)$, e.g. in Woodworking. This is because we set a budget of $10K$ nodes, and $\text{IW}(2)$ may exhaust the budget before finding the goal. We observe that $\text{IHIW}(1,1)$ outperforms $\text{IW}(1)$ in all but five domains: Barman, OpenStacks, Parking, Scannalyzer and Woodworking. Compared to $\text{IW}(2)$, $\text{IHIW}(1,1)$ covers more or the same number of instances in 24 out of 36 domains. In 12 cases the average number of nodes per solved instance is lower in $\text{IHIW(1,1)}$ than in $\text{IW}(2)$, and in 18 cases IHIW solved it faster. Note that Table~\ref{tab:pddl} only reports the average time for solved instances. Thus, we may find that IHIW is quicker than $\text{IW}(2)$ even when solving more instances.

With these results we can conclude that $\text{HIW}(1,1)$ can solve problems of width $2$ in practice, and that Algorithm~\ref{algo:findAbstract} is a good approach to identify promising high-level features. Finally, we can state that $\text{IHIW}(1,1)$ is an efficient alternative to $\text{IW}(2)$.

\section{Pixel-Based Testbeds}
In this section, we test our approach, $\pi\text{-HIW}$, in pixel-based gridworld environments and Atari games. We use two levels of abstraction: the high-level planner is Count-based Rollout IW (Algorithm \ref{algo:countbasedRIW}) and the low-level planner is $\pi\text{-IW+}$ (i.e., Rollout IW guided by the current policy estimate). The set of abstract features $\phi_h(s)$ consists of a discretization of the image, similar to the one used in Go-Explore~\cite{ecoffet2019go,Ecoffet_2021}, where the image is divided into tiles and the mean pixel value of each tile is taken as the feature value. Usually, this is further quantized into a smaller subset (e.g. $8$ pixel values). For the low-level set of features, we follow the methodology of \citet{junyent2019} and define $\phi_\ell(s)$ as the boolean discretization of $z(s)$, where $z$ is the last layer of the neural network representing $\widehat\pi_\theta$.

\begin{figure}[t]
\begin{center}
\includegraphics[width=0.25\columnwidth]{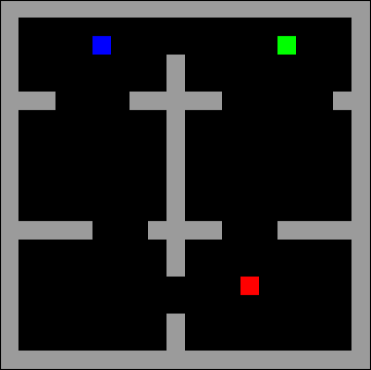}
\hspace{0.15\columnwidth}
\includegraphics[width=0.25\columnwidth]{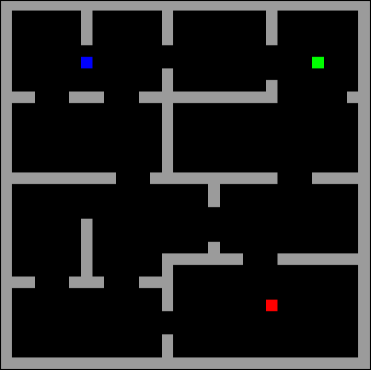}    
\end{center}
\caption{Snapshot of the two gridworld environments. Colors blue, red, green and gray represent the agent, key, door, and walls, respectively. The optimal policy takes 36 and 62 steps for the small (left) and large (right) tasks, respectively.}
\label{fig:envs}
\end{figure}

\begin{table*}[hbt!]
\centering

\begin{tabular}{l|c|ccc|ccc|ccc}
\multirow{2}{*}{Domain} & \multirow{2}{*}{I} & \multicolumn{3}{c|}{$\text{IW}(1)$} & \multicolumn{3}{c|}{$\text{IW}(2)$} & \multicolumn{3}{c}{$\text{IHIW}(1,1)$} \\ & & Solved & Nodes & Time & Solved & Nodes & Time & Solved & Nodes & Time \\
\hline
8puzzle & 32 & 40.6 & 34 & 0.00 & \textbf{100} & 475 & 0.04 & \textbf{100} & \textcolor{blue}{137} & \textcolor{blue}{0.01} \\
Barman & 232 & \textbf{9.1} & 215 & 0.02 & \textbf{9.1} & 215 & 0.13 & \textbf{9.1} & 215 & \textcolor{blue}{0.02} \\
Blocks World & 302 & 37.4 & 91 & 0.01 & 79.5 & 1696 & 0.23 & \textbf{96.4} & \textcolor{blue}{869} & \textcolor{blue}{0.06} \\
Cybersecurity & 86 & 65.1 & 64 & 0.01 & 65.1 & 64 & 0.22 & \textbf{67.4} & 158 & \textcolor{blue}{0.02} \\
Depots & 189 & 10.6 & 494 & 0.28 & 23.8 & 2393 & 1.58 & \textbf{28.0} & \textcolor{blue}{2268} & \textcolor{blue}{0.97} \\
Driverlog & 259 & 44.0 & 996 & 0.12 & 53.3 & 1249 & 0.18 & \textbf{62.9} & \textcolor{blue}{1085} & \textcolor{blue}{0.11} \\
Elevator & 510 & 0.0 & - & - & 11.4 & 5875 & 1.38 & \textbf{16.9} & \textcolor{blue}{4752} & 1.79 \\
Ferry & 8 & 0.0 & - & - & \textbf{100} & 10 & 0.00 & \textbf{100} & 11 & 0.00 \\
Floortile & 538 & 96.3 & 515 & 0.04 & 93.5 & 1115 & 0.63 & \textbf{99.3} & \textcolor{blue}{567} & \textcolor{blue}{0.04} \\
Freecell* & 68 & 8.8 & 192 & 0.14 & \textbf{22.1} & 3558 & 4.00 & 19.1 & 504 & 0.48 \\
Grid & 19 & 5.3 & 2 & 0.00 & \textbf{36.8} & 2071 & 6.45 & 15.8 & 1244 & 2.51 \\
Gripper & 460 & 0.0 & - & - & \textbf{100} & 3355 & 1.70 & \textbf{100} & \textcolor{blue}{2140} & \textcolor{blue}{0.36} \\
Logistics & 249 & 18.1 & 2 & 0.00 & \textbf{100} & 763 & 0.16 & 28.5 & 87 & 0.01 \\
Miconic & 2325 & 0.0 & - & - & 0.0 & - & - & \textbf{100} & 2751 & 0.24 \\
Mprime & 50 & 8.0 & 2 & 0.01 & 18.0 & 3316 & 0.75 & \textbf{20.0} & \textcolor{blue}{2600} & \textcolor{blue}{0.48} \\
Mystery & 45 & 8.9 & 2 & 0.01 & \textbf{37.8} & 1200 & 0.57 & 31.1 & 1903 & 0.37 \\
NoMystery & 210 & 0.0 & - & - & \textbf{80.0} & 1917 & 1.61 & 24.8 & 1487 & 1.22 \\
OpenStacks* & 455 & \textbf{0.0} & - & - & \textbf{0.0} & - & - & \textbf{0.0} & - & - \\
OpenStacksIPC6 & 1230 & 5.1 & 176 & 0.20 & \textbf{14.2} & 2637 & 11.46 & 13.8 & 2332 & 0.37 \\
PSRsmall & 316 & 89.9 & 2 & 0.00 & 92.1 & 2 & 0.00 & \textbf{94.0} & 3 & \textcolor{blue}{0.00} \\
ParcPrinter & 990 & 85.6 & 195 & 0.01 & 84.6 & 695 & 0.63 & \textbf{92.0} & \textcolor{blue}{464} & \textcolor{blue}{0.03} \\
Parking & 540 & \textbf{66.3} & 2770 & 2.28 & 65.2 & 2963 & 5.79 & \textbf{66.3} & \textcolor{blue}{2770} & \textcolor{blue}{2.27} \\
Pegsol & 990 & 92.6 & 4 & 0.00 & \textbf{100} & 9 & 0.01 & 97.8 & 7 & 0.00 \\
Pipes-NonTan & 259 & 45.6 & 299 & 0.08 & 55.6 & 1937 & 0.85 & \textbf{57.5} & \textcolor{blue}{683} & \textcolor{blue}{0.17} \\
Rovers* & 488 & 31.6 & 2520 & 0.37 & 23.2 & 2504 & 1.59 & \textbf{35.2} & 2576 & \textcolor{blue}{0.37} \\
Satellite* & 1324 & 5.7 & 367 & 0.19 & 7.2 & 675 & 0.23 & \textbf{7.9} & 1433 & \textcolor{blue}{0.22} \\
Scanalyzer & 648 & \textbf{99.1} & 370 & 0.29 & 96.6 & 322 & 0.66 & \textbf{99.1} & 370 & \textcolor{blue}{0.28} \\
Sokoban & 154 & 35.1 & 37 & 0.01 & \textbf{74.0} & 1049 & 5.36 & 40.3 & 84 & 0.01 \\
Storage & 240 & \textbf{100} & 327 & 1.87 & \textbf{100} & 1035 & 15.76 & \textbf{100} & \textcolor{blue}{327} & \textcolor{blue}{1.88} \\
Tpp* & 118 & 0.0 & - & - & \textbf{44.9} & 3313 & 26.01 & 35.6 & 1476 & 0.19 \\
Transport & 330 & 0.0 & - & - & 11.8 & 3765 & 1.20 & \textbf{18.5} & 4230 & 1.96 \\
Trucks & 345 & 0.0 & - & - & \textbf{11.6} & 5158 & 0.77 & 1.7 & 3342 & 0.47 \\
Visitall & 21880 & \textbf{100} & 2918 & 1.83 & 16.9 & 2912 & 1.34 & \textbf{100} & 2918 & 1.83 \\
Woodworking & 1801 & \textbf{91.6} & 1110 & 0.29 & 88.3 & 1063 & 3.43 & \textbf{91.6} & 1110 & \textcolor{blue}{0.29} \\
Zeno & 219 & 21.0 & 10 & 0.00 & \textbf{36.5} & 1740 & 0.18 & 29.2 & 1035 & 0.10 \\

\hline
 & & 7 & & & 17 & & & 24 & \textcolor{blue}{12} & \textcolor{blue}{18}

\end{tabular}
\caption{Comparison between $\text{IW}(1)$, $\text{IW}(2)$ and $\text{IHIW}(1,1)$ in different classical planning domains. Column I shows the number of single goal instances. In domains with an asterisk not all available instances were evaluated due to time or memory constraints. In columns 3-11 we show, for each algorithm, the coverage in percentage, the average amount of expanded nodes, and the average time in seconds. Nodes and time values only take into account solved instances. All algorithms have a planning budget of 10,000 nodes. Best coverage in bold, IHIW times or nodes that are lower than the ones of $\text{IW}(2)$ are shown in blue.}
\label{tab:pddl}
\end{table*}

\begin{figure*}[!ht]
\begin{center}
\includegraphics[width=0.47\linewidth]{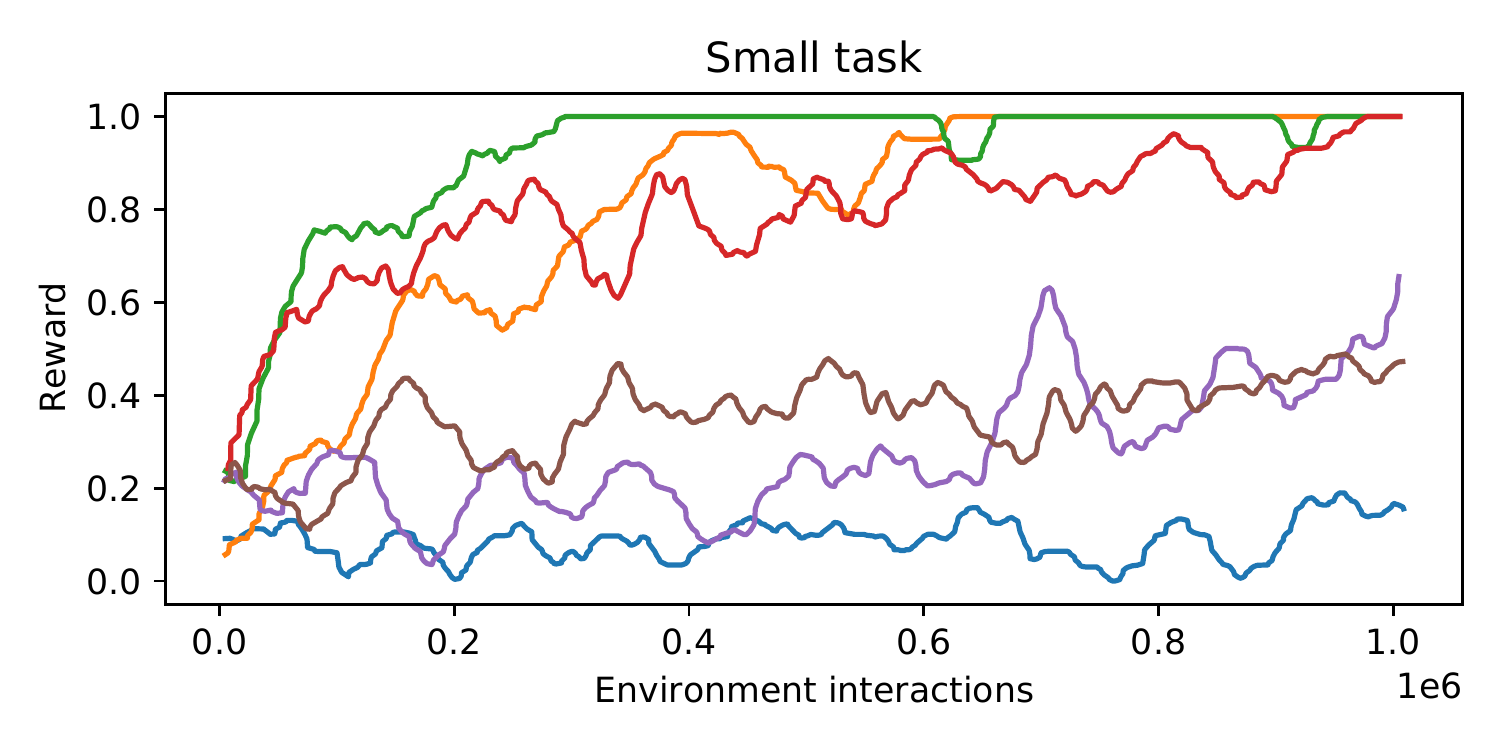}
\includegraphics[width=0.47\linewidth]{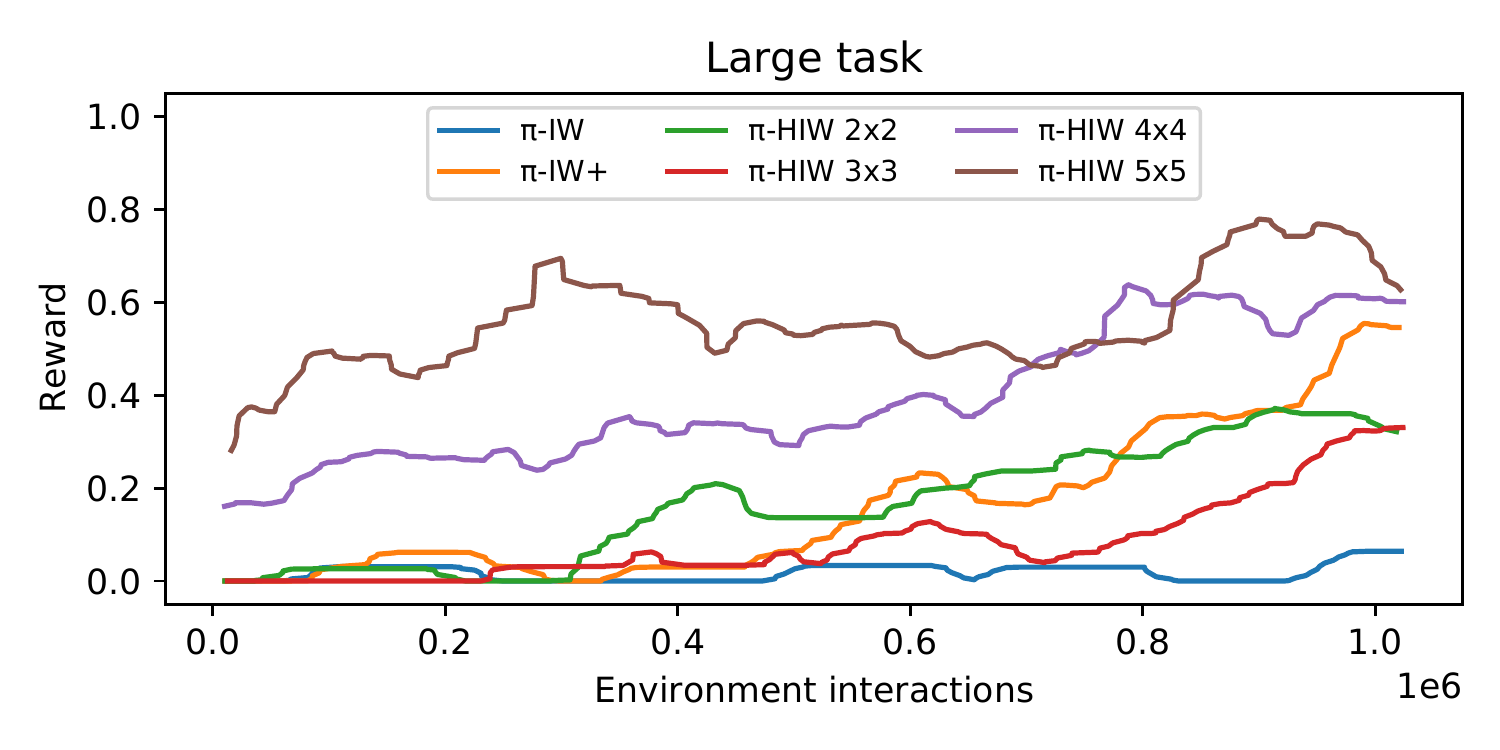}
\end{center}
\caption{Comparison between $\pi\text{-IW}$, $\pi\text{-IW+}$, and $\pi\text{-HIW}(1,1)$ in the small and large gridworld environments.}
\label{fig:grid-plots}
\end{figure*}

\begin{figure*}[t]
\begin{center}
\includegraphics[width=\textwidth]{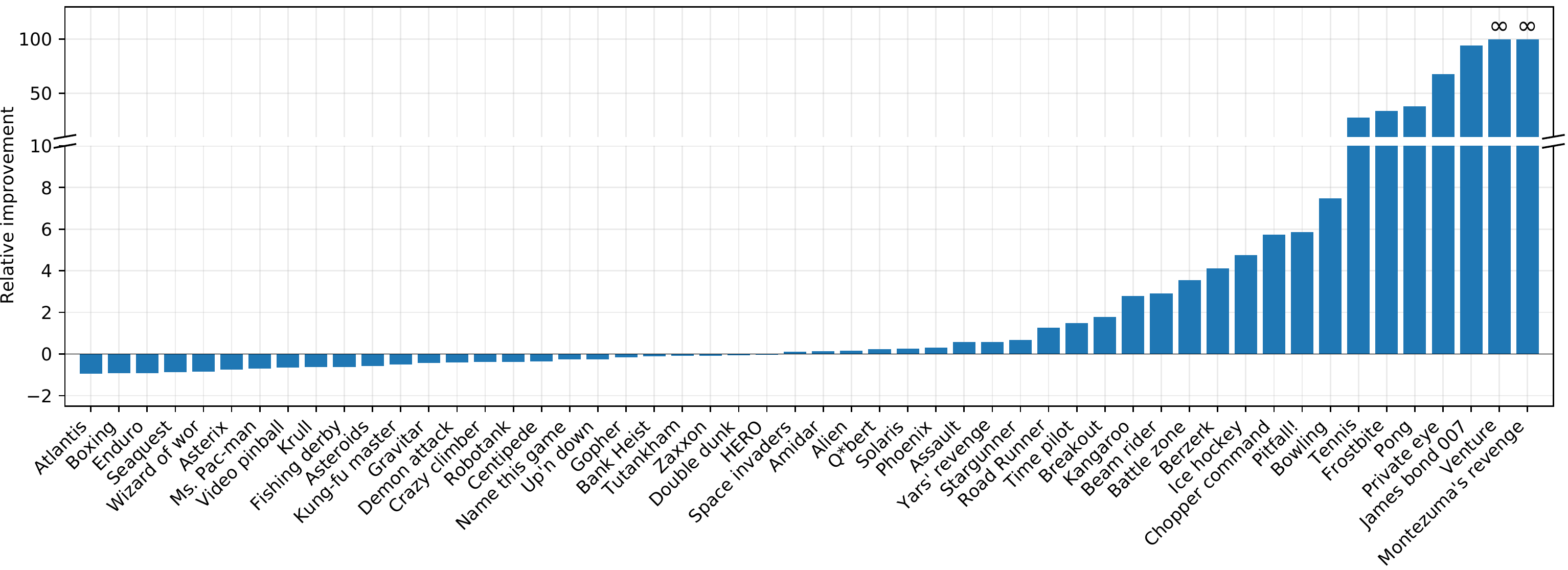}
\end{center}
\caption{Relative improvement of $\pi\text{-HIW}$ over $\pi\text{-IW}$ in Atari, $(s_{\pi\text{-HIW}}-s_{\text{random}})/(s_{\pi\text{-IW}}-s_{\text{random}})-1$, where 
$s_{\pi\text{-IW}}$ and $s_{\pi\text{-HIW}}$ are the scores of
the flat and hierarchical versions, respectively, and $s_{\text{random}}$ is the score of a random agent \cite{wang2016dueling}.
For Montezuma and Venture, the relative improvement is $\infty$, since $\pi\text{-IW}$ has $0$ score. Skiing is not shown since $s_{\pi\text{-IW}} < s_{\text{random}}$.
}
\label{fig:atari-all}
\end{figure*}

\subsection{Gridworld Environments}
We test our algorithm in two gridworld environments with sparse rewards (Figure \ref{fig:envs}). The agent (blue) has to pick up the key (red) and open the door (green), avoiding walls (gray). The agent is rewarded with $+1$ only when the door is reached while holding the key. Any other state has a reward of $0$, except if the agent hits a wall, in which case the episode terminates with a reward of $-1$. We also end the episode after $200$ and $500$ steps for the small and large environment, respectively. The observation is a $84\times 84\times 3$ image and possible actions are \{no-op, up, down, left, right\}. The setting is similar to the one of \citet{junyent2019}, but with larger environments and therefore sparser rewards.

We compare our hierarchical approach, $\pi\text{-HIW}(1,1)$, to two baselines: $\pi\text{-IW}$, and our modified version $\pi\text{-IW+}$ that uses a value estimate and the subtree size for tie-breaking. For the latter, we use a temperature of $1$ to generate $\pi^{\textnormal{counts}}$. In order to bound the memory used by the planner, we set a maximum of $500$ nodes that we keep in memory per step. The visitation count temperature used by the high-level planner (Algorithm \ref{algo:countbasedRIW}) is set to $0.005$. All other hyperparameters are the same as in \citet{junyent2019}.

Figure \ref{fig:grid-plots} shows results for both environments. We observe that $\pi\text{-IW}$ does not perform well, obtaining a reward close to zero in both environments. $\pi\text{-IW+}$ takes advantage of the value function and the tie-breaking counts and learns to solve the first task, while achieving a mean score of $0.5$ for the second one in $10^6$ interactions with the environment. For the hierarchical version, which also includes the aforementioned modifications, we report results of $\pi\text{-HIW}(1,1)$ using different number of tiles in $\phi_h(s)$, and $256$ values per tile. We observe how, for the smaller task, 2x2 tiles is enough to get a good performance, similar to the baseline $\pi\text{-IW+}$, and the performance degrades when increasing the number of tiles. In the larger task, $\pi\text{-HIW}(1,1)$ outperforms the baseline, but it needs at least 4x4 tiles to perform well.

\subsection{Atari Games}
We finish this section with a set of experiments using the Atari simulator. In this case, we do not optimize the hyper-parameters and define $F_h$ using $32$ pixels values and $8 \times 11$ tiles.
Moreover, we use width $w_h=n=|F_h|$ at the high level, i.e., $\pi\text{-HIW}(n,1)$. Even though $\text{IW}(n)$ explores the entire high-level state space, there is a single combination of $n$ features, which makes the novelty check efficient. In the original IW algorithm, $\text{IW}(n)$ is equivalent to a breadth-first search without state duplicates. Nevertheless, we use Count-Based Rollout IW, described in Algorithm \ref{algo:countbasedRIW}. With this, we aim to achieve effective widths larger than $2$.

\begin{figure}[t]
\begin{center}
\includegraphics[width=1.02\columnwidth]{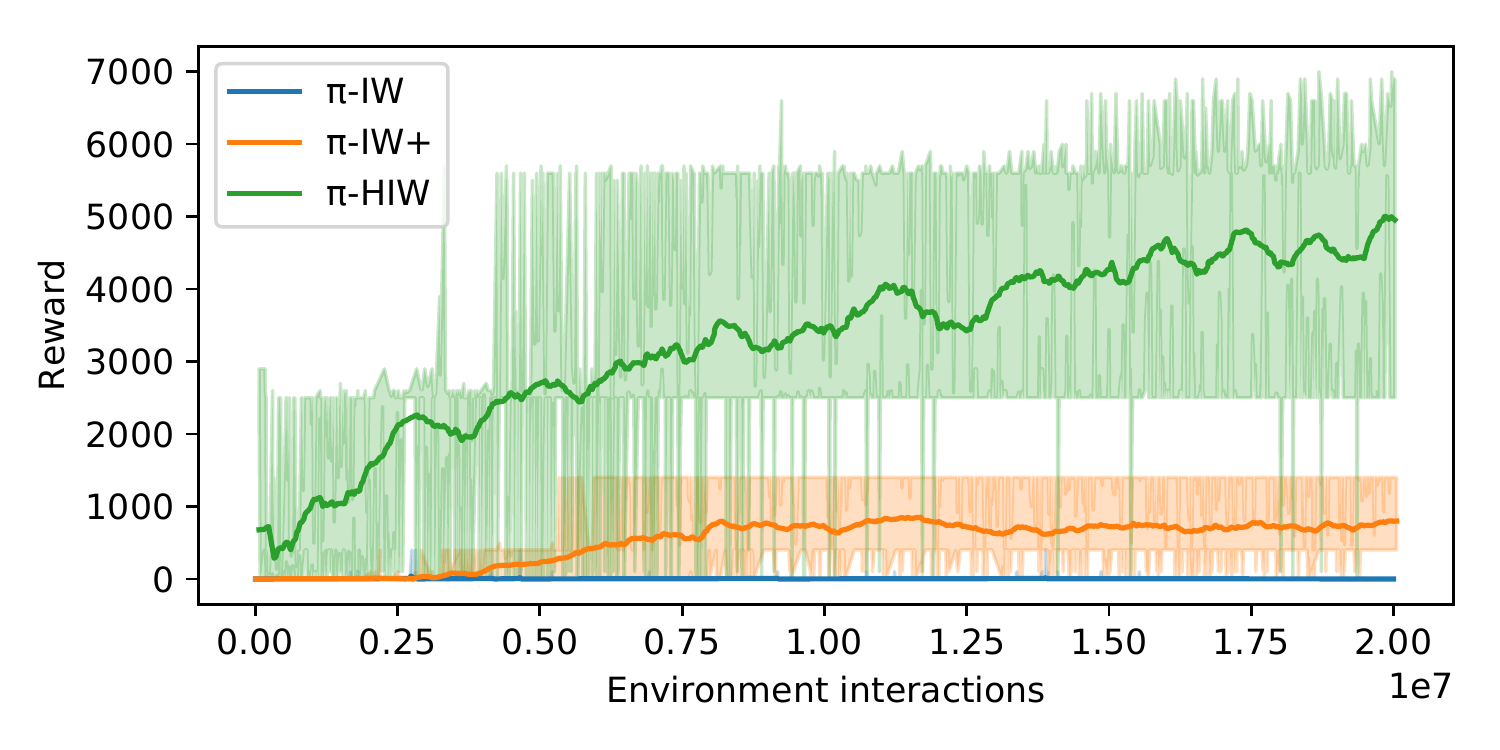}
\end{center}
\caption{Performance of $\pi\text{-IW}$, $\pi\text{-IW+}$ and $\pi\text{-HIW}$ in Montezuma's Revenge. Average over $5$ runs with different random seeds. Shades show the maximum and minimum value.}
\label{fig:atari-montezuma}
\end{figure}

Figure~\ref{fig:atari-all} shows a comparison between $\pi\text{-HIW}(n, 1)$ and $\pi\text{-IW}$ using the same setup as in~\citet{junyent2019}, but half the budget of simulator interactions. The full results are given in Table \ref{tab:atari}.
We observe that $\pi\text{-HIW}$ improves over its predecessor $\pi\text{-IW}$ in 28 games. 
Interestingly, games consisting of an agent moving in a fixed background present the best results e.g., James Bond, Private Eye, Pong, Frostbite, etc. %
Within this type of games, $\pi\text{-HIW}$ remarkably achieves a positive score in hard exploration games such as Montezuma's Revenge and Venture, a score not yet reported for any width-based planner.
Figure~\ref{fig:atari-montezuma} shows the learning curve in the game of Montezuma's Revenge. 
We also see an improvement in games with a moving background where the agent stays at a fixed position, for instance in Battle zone, Beam Rider, or Time Pilot. 
These results confirm that $\pi\text{-HIW}$  benefits from the state abstractions provided by a simple down-sample of the image. 

\section{Conclusions}

We presented a novel hierarchical approach to width-based planning. Our method uses different feature mappings to create several levels of abstraction, allowing different search algorithms at different levels of the planning hierarchy. Specifically, we propose to use Iterated Width at two levels, resulting in the hierarchical search algorithm $\text{HIW}(w_h, w_\ell)$. We show that $\text{HIW}(w_h, w_\ell)$ can solve problems of width $w_h + w_\ell$ with the right choice of high-level features. Experiments in planning benchmarks show that an incremental version of $\text{HIW}(1,1)$ is competitive with $\text{IW}(2)$, solving single-goal instances using less time or nodes. When combined with a policy learning scheme, HIW achieves a positive score in hard exploration Atari games such as Montezuma's Revenge. For future work, a promising approach is to explore different combinations of search algorithms at different levels of the hierarchy.

\section{Acknowledgements}
V. G\'omez has received funding from ``La Caixa'' Foundation (100010434), under the agreement LCF /PR/PR16/51110009 and is supported by the Ramon y Cajal program RYC-2015-18878
(AEI/MINEICO/FSE,UE). A. Jonsson is partially supported by Spanish grants PID2019-108141GB-I00 and PCIN-2017-082.

\bibstyle{abbrv}
\bibliography{bibliography.bib}

\onecolumn
\newpage
\appendix
\section{Proof of Theorem~\ref{thm:N}}

Here we prove Theorem~\ref{thm:N}, which states that for $n$ features with bounded domain size $d$, the maximum number of novel nodes expanded by IW($w$), $0\leq w<n$, is given by
\[
    N(n,d,w) = \sum_{k=0}^{w}\left[ \binom{n-1-k}{w-k} d^k (d-1)^{w-k} \right].
\]
The proof is by induction on pairs of integers $(n,w)$. The base case is given by $(n,0)$, in which case we have
\[
N(n,d,0) = \sum_{k=0}^{0}\left[ \binom{n-1-k}{0-k} d^0 (d-1)^{0-k} \right] = \binom{n-1}{0} d^0 (d-1)^0 = 1.
\]
For $(n,w)$ such that $0<w<n-1$, by hypothesis of induction we assume that Theorem~\ref{thm:N} holds for $(n-1,w-1)$ and $(n-1,w)$. Applying the recursive definition yields
\begin{align*}
	N(n,d,w) &= (d-1) N(n-1,d,w-1) + N(n-1,d,w)\\
	&= (d-1)\sum_{k=0}^{w-1}\left[ \binom{n-2-k}{w-1-k} d^k (d-1)^{w-1-k} \right] + \sum_{k=0}^w\left[ \binom{n-2-k}{w-k} d^k (d-1)^{w-k} \right]\\
	&= \sum_{k=0}^{w-1}\left[\left( \binom{n-2-k}{w-1-k} + \binom{n-2-k}{w-k} \right) d^k (d-1)^{w-k} \right] + \binom{n-2-w}{0} d^w (d-1)^0\\
	&= \sum_{k=0}^{w-1}\left[ \binom{n-1-k}{w-k} d^k (d-1)^{w-k} \right] + \binom{n-1-w}{0} d^w (d-1)^0\\
	&= \sum_{k=0}^w \left[ \binom{n-1-k}{w-k} d^k (d-1)^{w-k} \right].
\end{align*}
Here, we used the identities $\binom{n-1}{m-1}+\binom{n-1}{m}=\binom{n}{m}$, $0<m<n$, and $\binom{n}{0}=1=\binom{n+1}{0}$.

For $(n,w)$ such that $w=n-1$, by hypothesis of induction we assume that Theorem~\ref{thm:N} holds for $(n-1,w-1)$. Applying the recursive definition yields
\begin{align*}
	N(n,d,w) &= (d-1) N(n-1,d,w-1) + N(n-1,d,w)\\
	&= (d-1)\sum_{k=0}^{w-1}\left[ \binom{n-2-k}{w-1-k} d^k (d-1)^{w-1-k} \right] + d^w\\
	&= \sum_{k=0}^{w-1}\left[ \binom{n-1-k}{w-k} d^k (d-1)^{w-k} \right] + \binom{n-1-w}{0} d^w (d-1)^0\\
	&= \sum_{k=0}^w \left[ \binom{n-1-k}{w-k} d^k (d-1)^{w-k} \right].
\end{align*}
Here, we used the definition $N(n-1,d,w)=N(w,d,w)=d^w$ and the identity $\binom{n}{n} = 1 = \binom{n+1}{n+1}$, which is applicable since $(w-1-k)=(n-1-1-k)=(n-2-k)$.

To obtain a compact upper bound on $N(n,d,w)$, we can write
\begin{align*}
    N(n,d,w) &= \sum_{k=0}^{w}\left[ \binom{n-1-k}{w-k} d^k (d-1)^{w-k} \right]\\
	&\leq d^w \sum_{k=0}^{w}\binom{n-1}{w-k} = d^w \sum_{k=0}^{w}\binom{n-1}{k} \leq d^wn^w = (nd)^w,
\end{align*}
where the last inequality follows from the binomial theorem.

\begin{table}[ht]
\centering
\begin{tabular}{c|c|cccccc|c}
    & \textbf{($f_0$, $f_1$, $f_2$, $f_3$)} & $f_0f_1$ & $f_0f_2$ & $f_0f_3$ & $f_1f_2$ & $f_1f_3$ & $f_2f_3$ & \# novel states \\
    \hline
    \multirow{8}{*}{$S_{\neg{f_0}}$} & \textbf{(\textcolor{gray}{0}, 0, 0, 0)} & \textbf{\textcolor{gray}{00}} & \textbf{\textcolor{gray}{00}} & \textbf{\textcolor{gray}{00}} & \textbf{00} & \textbf{00} & \textbf{00} & \multirow{8}{*}{$N(n-1, d, w)$}\\
    & \textbf{(\textcolor{gray}{0}, 0, 0, 1)} & \textcolor{gray}{00} & \textcolor{gray}{00} & \textbf{\textcolor{gray}{01}} & 00 & \textbf{01} & \textbf{01} \\
    & \textbf{(\textcolor{gray}{0}, 0, 1, 1)} & \textcolor{gray}{00} & \textbf{\textcolor{gray}{01}} & \textcolor{gray}{01} & \textbf{01} & 01 & \textbf{11} \\
    & \textbf{(\textcolor{gray}{0}, 0, 1, 0)} & \textcolor{gray}{00} & \textcolor{gray}{01} & \textcolor{gray}{00} & 01 & 00 & \textbf{10} \\
    & \textbf{(\textcolor{gray}{0}, 1, 1, 0)} & \textbf{\textcolor{gray}{01}} & \textcolor{gray}{01} & \textcolor{gray}{00} & \textbf{11} & \textbf{10} & 10 \\
    & \textbf{(\textcolor{gray}{0}, 1, 1, 1)} & \textcolor{gray}{01} & \textcolor{gray}{01} & \textcolor{gray}{01} & 11 & \textbf{11} & 11 \\
    & \textbf{(\textcolor{gray}{0}, 1, 0, 1)} & \textcolor{gray}{01} & \textcolor{gray}{00} & \textcolor{gray}{01} & \textbf{10} & 11 & 01 \\
    & (\textcolor{gray}{0}, 1, 0, 0) & \textcolor{gray}{01} & \textcolor{gray}{00} & \textcolor{gray}{00} & 10 & 10 & 00 \\
    \hline
    \multirow{8}{*}{$S_{f_0}$} & \textbf{(\textcolor{gray}{1}, 1, 0, 0)} & \textbf{\textcolor{gray}{0}1} & \textbf{\textcolor{gray}{0}0} & \textbf{\textcolor{gray}{0}0} & \textcolor{gray}{10} & \textcolor{gray}{10} & \textcolor{gray}{00} & \multirow{8}{*}{$(d-1) \cdot N(n-1, d, w-1)$} \\
    & \textbf{(\textcolor{gray}{1}, 1, 0, 1)} & \textcolor{gray}{0}1 & \textcolor{gray}{0}0 & \textbf{\textcolor{gray}{0}1} & \textcolor{gray}{10} & \textcolor{gray}{11} & \textcolor{gray}{01} \\
    & \textbf{(\textcolor{gray}{1}, 1, 1, 1)} & \textcolor{gray}{0}1 & \textbf{\textcolor{gray}{0}1} & \textcolor{gray}{0}1 & \textcolor{gray}{11} & \textcolor{gray}{11} & \textcolor{gray}{11} \\
    & (\textcolor{gray}{1}, 1, 1, 0) & \textcolor{gray}{0}1 & \textcolor{gray}{0}1 & \textcolor{gray}{0}0 & \textcolor{gray}{11} & \textcolor{gray}{10} & \textcolor{gray}{10} \\
    & \textbf{(\textcolor{gray}{1}, 0, 1, 0)} & \textbf{\textcolor{gray}{0}0} & \textcolor{gray}{0}1 & \textcolor{gray}{0}0 & \textcolor{gray}{01} & \textcolor{gray}{00} & \textcolor{gray}{10} \\
    & (\textcolor{gray}{1}, 0, 1, 1) & \textcolor{gray}{0}0 & \textcolor{gray}{0}1 & \textcolor{gray}{0}1 & \textcolor{gray}{01} & \textcolor{gray}{01} & \textcolor{gray}{11} \\
    & (\textcolor{gray}{1}, 0, 0, 1) & \textcolor{gray}{0}0 & \textcolor{gray}{0}0 & \textcolor{gray}{0}1 & \textcolor{gray}{00} & \textcolor{gray}{01} & \textcolor{gray}{01} \\
    & (\textcolor{gray}{1}, 0, 0, 0) & \textcolor{gray}{0}0 & \textcolor{gray}{0}0 & \textcolor{gray}{0}0 & \textcolor{gray}{00} & \textcolor{gray}{00} & \textcolor{gray}{00} \\
\end{tabular}
\caption{List of all possible $d^n$ states, with $n=4$ features and domain size $d=2$. We list them in Gray code (i.e., only one bit changing at a time) and consider this as the order of expansion of $\text{IW}(2)$, ensuring the worst case scenario where the maximum number of states are considered novel. The third column shows the tuple combinations taken into account in the novelty test. Novel states and the feature tuples that make them novel are shown in bold. The first 7 states are novel, in part, due to tuples of features $f_1$, $f_2$, and $f_3$, and therefore belong to $S_{\neg{f_0}}$ while the last 4 novel states are novel exclusively due to tuples containing $f_0$, and belong to $S_{f_0}$. Values of $f_0$ and tuples that are irrelevant for the novelty test in subsets $S_{\neg{f_0}}$ and $S_{f_0}$ are shown in gray.}
\label{tab:N-rec-example}
\end{table}
\begin{table*}[t]
\centering
\begin{tabular}{lrrr}
Game & $\pi\text{-IW}(1)$ & $\pi\text{-IW}(1)\text{+}$ & $\pi\text{-HIW}(n,1)$ \\
\hline
Alien & 3969.78 & 2585.77 & \textbf{4609.18} \\
Amidar & 950.45 & 374.20 & \textbf{1076.17} \\
Assault & 1574.91 & 922.30 & \textbf{2344.28} \\
Asterix & \textbf{346409.11} & 247063.36 & 90017.25 \\
Asteroids & 1368.55 & \textbf{1490.87} & 990.95 \\
Atlantis & 106212.63 & \textbf{143177.73} & 17539.22 \\
Bank Heist & \textbf{567.16} & 256.29 & 501.68 \\
Battle zone & 69659.40 & 30848.95 & \textbf{309137.79} \\
Beam rider & 3313.11 & 8428.96 & \textbf{11931.41} \\
Berzerk & 1548.23 & 960.03 & \textbf{7417.26} \\
Bowling & 26.28 & \textbf{78.18} & 50.09 \\
Boxing & \textbf{99.88} & 88.19 & 6.81 \\
Breakout & 92.07 & 107.64 & \textbf{252.88} \\
Centipede & 126488.35 & \textbf{141070.19} & 80685.48 \\
Chopper command & 11187.44 & 3431.74 & \textbf{70787.12} \\
Crazy climber & \textbf{161192.01} & 138648.58 & 102205.99 \\
Demon attack & 26881.13 & \textbf{35022.64} & 16007.64 \\
Double dunk & \textbf{4.68} & -16.80 & 3.51 \\
Enduro & \textbf{506.59} & 63.83 & 44.47 \\
Fishing derby & \textbf{8.89} & -28.02 & -53.76 \\
Frostbite & 270.00 & 1636.51 & \textbf{7242.60} \\
Gopher & \textbf{18025.91} & 7061.76 & 15001.18 \\
Gravitar & \textbf{1876.80} & 1532.33 & 1154.01 \\
HERO & \textbf{36443.73} & 22097.39 & 36231.21 \\
Ice hockey & -9.66 & -4.02 & \textbf{-2.36} \\
James bond 007 & 43.20 & 205.91 & \textbf{1380.13} \\
Kangaroo & 1847.46 & 2918.98 & \textbf{6861.57} \\
Krull & 8343.30 & \textbf{13014.77} & 4121.81 \\
Kung-fu master & \textbf{41609.03} & 24871.94 & 20680.65 \\
Montezuma's revenge & 0.00 & 810.49 & \textbf{5275.89} \\
Ms. Pac-man & \textbf{14726.33} & 5916.86 & 4523.47 \\
Name this game & 12734.85 & \textbf{18167.55} & 9977.12 \\
Phoenix & 5905.12 & \textbf{7647.67} & 7508.63 \\
Pitfall! & -214.75 & \textbf{-2.46} & -128.82 \\
Pong & -20.42 & \textbf{2.14} & -9.70 \\
Private eye & 452.40 & 1766.13 & \textbf{29548.76} \\
Q*bert & 32529.60 & 23337.90 & \textbf{40449.72} \\
Road Runner & 38764.81 & 43813.29 & \textbf{87953.53} \\
Robotank & \textbf{15.66} & 9.68 & 10.63 \\
Seaquest & \textbf{5916.05} & 559.28 & 867.51 \\
Skiing & -19188.32 & \textbf{-13852.04} & -15417.86 \\
Solaris & 3048.78 & 1832.93 & \textbf{3524.69} \\
Space invaders & 2694.09 & 1622.49 & \textbf{2946.18} \\
Stargunner & 1381.24 & 1642.82 & \textbf{1864.64} \\
Tennis & -23.67 & \textbf{-8.26} & -20.00 \\
Time pilot & 16099.92 & 11126.86 & \textbf{34610.25} \\
Tutankham & \textbf{216.67} & 181.44 & 199.06 \\
Up'n down & \textbf{107757.51} & 59497.75 & 80991.07 \\
Venture & 0.00 & \textbf{15.68} & 10.73 \\
Video pinball & \textbf{514012.51} & 387308.60 & 184720.01 \\
Wizard of wor & \textbf{76533.18} & 30383.68 & 12027.43 \\
Yars' revenge & 102183.67 & 64544.51 & \textbf{159496.20} \\
Zaxxon & \textbf{22905.73} & 10159.01 & 21135.58 \\
\hline
\# best & 19 & 14 & 21 \\

\end{tabular}
\caption{Comparison of $\pi\text{-IW}(1)$, $\pi\text{-IW}(1)\text{+}$ and $\pi\text{-HIW}(n,1)$ over 53 Atari games. Best score given in bold. Results for Freeway are not included because the simulator was excessively slow compared to other games.}
\label{tab:atari}
\end{table*}

\end{document}